\documentclass[twoside]{article}
\usepackage{amsmath,url,amsthm,amssymb,amsbsy,fancybox}
\usepackage{algorithmic,algorithm}
\usepackage{subfigure,graphicx}
\usepackage{placeins}
\usepackage[numbers,sort]{natbib}

\usepackage[accepted]{aistats2012}

\runningauthor{Balakrishnan, Rinaldo, Sheehy, Singh, Wasserman}

\let\hat\widehat

\newcommand{\vol}{\mathrm{vol}}

\newenvironment{packed_enum}{
\begin{enumerate}
\setlength{\itemsep}{1pt}
\setlength{\parskip}{0pt}
\setlength{\parsep}{0pt}
}{\end{enumerate}}

\newcommand{\simpcomp}{\cal{K}}
\newcommand{\bdy}{\partial}
\newcommand{\im}{\mathrm{im}~}
\newcommand{\Z}{\mathbb{Z}}
\newcommand{\R}{\mathbb{R}}

\newtheorem{assumption}{Assumption}

\newtheorem{lemma}{Lemma}

\newtheorem{theorem}{Theorem}
\newtheorem{remark}{Remark}

\newenvironment{enum}{
\begin{enumerate}
  \setlength{\itemsep}{1pt}
  \setlength{\parskip}{0pt}
  \setlength{\parsep}{0pt}
}{\end{enumerate}}

\begin{document}

\twocolumn[

\aistatstitle{Minimax Rates for Homology Inference}

\aistatsauthor{Sivaraman Balakrishnan$^\dagger$ \And  Alessandro Rinaldo$^\dagger$ \And Don Sheehy$^{\dagger\dagger}$}
\aistatsauthor{Aarti Singh$^{\dagger}$ \And Larry Wasserman$^{\dagger}$}

\aistatsaddress{ $^{\dagger}$Carnegie Mellon University  \And $^{\dagger\dagger}$INRIA } ]

\begin{abstract}
Often, high dimensional data lie
close to a low-dimensional submanifold and it is of interest to understand the geometry
of these submanifolds. The homology groups of a manifold are important topological
invariants that provide an algebraic summary of 
the manifold. These groups
contain rich topological information, for instance, about the connected components, holes, tunnels and sometimes the dimension of the manifold.
In this paper, we consider the statistical
problem of estimating the homology of a manifold from noisy samples under
several different noise models.
We derive upper and lower bounds
on the minimax risk for this problem.
Our upper bounds are based on estimators which are constructed from 
a union of balls of appropriate radius
around carefully selected points.
In each case we establish complementary lower bounds
using Le Cam's lemma.
\end{abstract}

\vspace{-.5cm}
\section{Introduction}
Let $M$ be a 
$d$-dimensional manifold embedded in $\mathbb{R}^D$
where $d \leq D$.
The {\em homology groups}   ${\cal H}(M)$  of $M$ \cite{hatcher01} is an algebraic summary of the properties of $M$. The homology groups of a manifold describe its topological features
such as its connected components, holes, tunnels, etc.

In machine learning, there is much focus on clustering.
However, the clusters are only the zeroth order homology
and hence only scratch the surface of the topological information
in a dataset. Extracting information beyond clustering is known 
as topological data analysis.
It is worth emphasizing that the homology groups are 
topological invariants of a manifold that can be \emph{efficiently} computed \cite{plex,witness}.
Examples of applications of homology inference 
have been growing rapidly in the last few years. Homology inference has found application
in medical
imaging and neuroscience \cite{bubenik09, singh08}, sensor networks
\cite{ghrist07, silva07}, landmark-based shape data analyses
\cite{gamble10}, proteomics \cite{sacan07}, microarray analysis
\cite{dequeant08} and cellular biology \cite{kasson07}. The books by
\cite{afra05, edelsbrunner09, pascucci01} contain various
case studies in applications in fields ranging from computational biology to geophysics.

In this paper we study the problem of 
estimating the homology of a manifold $M$ from a noisy sample
$Y_1,\ldots, Y_n$.
Specifically, we bound the minimax risk
\begin{equation}
R_n\equiv \inf_{\hat{\cal H}}\sup_{Q\in {\cal Q}}
Q^n \Bigl(\hat{\cal H} \neq {\cal H}(M)\Bigr)
\end{equation}
where
the infimum is over \emph{all} estimators
$\hat{\cal H}$ of the homology of $M$ and the supremum is over
appropriately defined classes of distributions ${\cal Q}$ for $Y$.
Note that $0 \leq R_n \leq 1$ with $R_n=1$
meaning that the problem is hopeless.
Bounding the minimax risk is equivalent to
bounding the {\em sample complexity} of the best possible estimator, 
defined by
$n(\epsilon) = \min\bigl\{n:\ R_n \leq \epsilon\bigr\}$
where $0 < \epsilon < 1$.

\subsection{Related Work}
Other work on statistical homology includes that of Chazal et. al.
\cite{chazal2011} who show under certain conditions
the homology estimate of a manifold from a sample is stable under
 noise perturbation that is small in a 
Wasserstein sense.
Kahle \cite{Kahle20091658} studies
the homology of random geometric graphs and proves many threshold and central limit 
theorems for their homology.
Adler et. al. \cite{adler} study 
the homology induced by the level sets of certain Gaussian random fields. 
There is also a large literature on manifold denoising that 
focuses on aspects of the manifold not related to homology;
see for instance \cite{hein06manifold} and references therein.

Our upper bounds mainly generalize those in the work of Niyogi, Smale 
and Weinberger (henceforth NSW) \cite{niyogi2008,niyogi2010}.
They establish a general result showing that when \emph{all} the samples 
are dense in a thin region surrounding the manifold, a union of appropriately
sized balls around the samples can be used to construct an accurate estimate of the homology
with high probability. Under a variety of different noise models we will show that
even when \emph{all} the samples are not close to the manifold
it is possible to ``clean'' the samples (essentially removing those in regions of low-density)
and be left with samples which are dense in a thin region around the manifold.

In the case of additive noise with general noise distributions however,
we cannot expect too many samples to fall close to the manifold. We
will show that when the noise distribution is known one can use a
statistical deconvolution procedure to obtain a ``deconvolved
measure'' concentrated around the manifold from which we can in turn
draw a small number of samples and apply the cleaning procedure
described above to them.  Deconvolution has been extensively studied in the
statistical literature (see \cite{fan_deconvolution} and references therein). Most related to
our application is the work of Koltchinskii \cite{koltchinskii} who uses
deconvolution to estimate the dimension and cluster tree of a
distribution supported on a submanifold. We defer a detailed
comparison to Section \ref{sec:decon} after the necessary preliminaries
have been introduced.


To the best of our knowledge, ours is the first paper
to obtain lower and upper minimax bounds for the problem of
inferring the homology of a manifold.
There are a few existing results on upper bounds.
A summary of previous results and 
the results in this paper are in Table 1.

{\em Outline.}
In Section \ref{sec::model} we describe the statistical model.
In Section \ref{sec::homology} we give a brief description of homology.
In Section \ref{sec::prelim} we give an overview of our techniques.
We derive the minimax rates for the four noise settings in Section \ref{sec:rates}. Technical proofs are contained in the Appendix.

\begin{center}
\begin{table*}
\label{tab:summary}
\begin{tabular}{c||ccccc}
             & \multicolumn{5}{c}{{\bf Noise Model}} \\
             & {\bf Noiseless}   & {\bf Clutter}    & {\bf Tubular}    & {\bf Additive Gaussian} & {\bf General additive ($\tau$ fixed)} \\ \hline
{\bf Upper Bound} &  NSW        & This paper & NSW        & This paper & This paper \\
{\bf Lower Bound}  &  This paper & This paper & This paper &  This paper & This paper
\end{tabular}
\caption{Summary of our contributions}
\end{table*}
\end{center}

\vspace{-1cm}
\section{Statistical Model}
\label{sec::model}
We assume that the sample $\{ Y_1,\ldots, Y_n  \} \subset \mathbb{R}^D$ constitutes a set of ``noisy" observations of an unknown $d$-dimensional manifold $M$, with $d < D$, whose homology  we seek to estimate. The distribution of the sample depends on the properties of the manifold $M$ as well as on the type of sampling noise, which we  describe below by formulating various statistical models for sampling data from manifolds. 

{\bf Notation.}
We let $B^k_r(x)$ denote a $k$-dimensional ball of radius $r$ centered at $x$.
When $k=D$, we write
$B_r(x)$ instead of $B_r^D(x)$.
For any set $M$ and any $\sigma>0$ define
${\sf tube}_\sigma(M) = \bigcup_{x\in M}B_\sigma(x)$.
Let $v_k$ denote the volume of the $k$-dimensional unit ball.
Finally, for clarity we let
$c_1,c_2, \ldots, C_1,C_2,\ldots$
denote various positive constants whose value
can be different in different expressions. The constants will be specified in
the corresponding proofs.



\vspace{.1cm}
{\bf Manifold Assumptions.}
We assume that the unknown manifold $M$ is a $d$-dimensional smooth compact Riemannian 
manifold without boundary embedded in
the compact set ${\cal X} = [0,1]^D$. We further assume that the volume of the manifold 
is bounded from above by a constant which can depend on the dimensions $d, D$, 
i.e. we assume $\vol(M) \leq C_{D,d}$.
Compact $d$-dimensional manifolds without boundary typically reside in an ambient dimension $D > d$, an assumption we will make throughout this paper. The main regularity condition we impose on $M$ is that its {\em condition number} be not too small. 
The {\em condition number} $\kappa(M)$ (see \cite{niyogi2008}) is the largest number $\tau$
such that the open normal bundle about $M$ of radius $r$ 
is imbedded in $\mathbb{R}^D$ for every $r < \tau$.
For $\tau >0$ let
${\cal M} \equiv {\cal M}(\tau)$=$\Bigl\{M: \kappa(M) \geq \tau \Bigr\}$
denote the set of all such manifolds with condition number no smaller than $\tau$.
A manifold with large condition number 
does not come too
close to being self-intersecting.
We consider the collection $\mathcal{P} \equiv \mathcal{P}(\mathcal{M}) \equiv \mathcal{P}(\mathcal{M},a)$ of all probability distributions supported over manifolds $M$ in ${\cal M}$ having densities $p$ with respect to the volume form on $M$ uniformly bounded from below by a constant $a>0$, i.e. 
$0 < a \leq p(x) < \infty $ for all $x\in M$. For expositional clarity we treat $a$ as a \emph{fixed} constant although our upper and lower bounds match in their dependence on $a$.

\vspace{.1cm}
{\bf The Noise Models.}
We consider four noise models and, for each of them, we specify a class $\mathcal{Q}$ of probability distributions for the  sample. 

 {\it Noiseless.}
We observe data 
$Y_1,\ldots, Y_n \sim P$
where $P \in \mathcal{P}$. In this case, 
${\cal Q}  ={\cal Q}(\tau) = {\cal P}$.

{\it Clutter Noise.}
We observe data $Y_1,\ldots, Y_n$ from the mixture
 $Q=(1-\pi)U + \pi P$
where, $P \in \mathcal{P}$, $0\leq \pi \leq 1$ and
$U$ is a uniform distribution on ${\cal X}$. The points drawn from $U$ are called background clutter. Then
${\cal Q} = {\cal Q}(\pi,\tau) = \Bigl\{ Q = (1-\pi)U + \pi P:\ P\in {\cal P} \Bigr\}.$
Notice that $\pi=1$ reduces to the noiseless case.

{\it Tubular Noise.}
We observe
$Y_1,\ldots, Y_n \sim Q_{M,\sigma}$
where $Q_{M,\sigma}$ is uniform on a tube of size $\sigma$ around $M$.
In this case
${\cal Q} = {\cal Q}(\sigma, \tau) = \Bigl\{ Q_{M,\sigma}:\ M \in {\cal M}\Bigr\}.$

{\it Additive Noise.}
The  data are of the form
$Y_i = X_i + \epsilon_i$,
where $X_1,\ldots, X_n \sim P$, for some $P \in \mathcal{P}$, and
$\epsilon_1,\ldots,\epsilon_n$ are a sample from a noise distribution $\Phi$.
Note that $Q = P\star \Phi$, that is, $Q$ is the convolution of $P$ and $\Phi$.
We consider two cases:
\begin{packed_enum}
\item 
$\Phi$ is
a $D$-dimensional Gaussian with mean 
$(0,\ldots, 0)$ and covariance $\sigma^2 I$, with $\sigma \ll \tau$.
Define
${\cal Q} = {\cal Q}(\sigma, \tau) = \Bigl\{ Q=P \star \Phi:\ P \in {\cal P}\Bigr\}.$
\item 
$\Phi$ is any known noise distribution whose Fourier transform
is bounded away from $0$ but with the added restriction that we only consider manifolds
with $\tau$ being a fixed constant.
Then
${\cal Q} = {\cal Q}(\Phi) = \Bigl\{ Q=P \star \Phi:\ P \in {\cal P}_\tau\Bigr\}.$
where ${\cal P}_\tau$ is the subset of ${\cal P}$ comprised of distributions supported on manifolds $M$ with condition number at least as large as the  ${\it fixed}$ value $\tau$.
\end{packed_enum}
\vspace{-.2cm}

The noise model
used in
\cite{niyogi2010}
is to take the noise at any point to be only along the normal fibres;
this seems unnatural and we will not consider that model here.

In almost all of the distribution classes considered  we allow for $\tau$ to vanish as $n$ gets bigger, which is equivalent to letting the difficulty of the statistical problem increase with the sample size.
To this end, we will also analyze the quantity $\tau_n \equiv \tau_n(\epsilon)= \inf\{\tau:\ R_n \leq \epsilon\}$, which corresponds to the smallest condition number that
permits accurate estimation. We call this the \emph{resolution}.

\vspace{-.5cm}
\section{Homology}
\vspace{-.3cm}
\label{sec::homology}
Often in our paper we will use phrases like ``the homology of the union of balls around samples''. 
In this section we explain this usage  and discuss briefly \emph{simplicial homology} (see Hatcher (2001)
for a detailed treatment) and its computation.


The homology ${\cal H}$ of a
 space $S$ is a collection of groups that correspond to
topological features of $S$. 
In what follows, it might help the reader's
intuition to imagine that we are starting
with a dense sample of points $U$ on a manifold
and building a collection of simplices from these points.
The union of balls $\bigcup_{y\in U}B_\epsilon(y)$ gives a geometric approximation 
to the underlying manifold. This is however a continuous (infinite) collection of points.
To make computation tractable we need to be able to reduce the computation
of homology from a continuous space to its discretization.
The \v Cech complex (a particular \emph{simplicial complex}, see Figure \ref{fig::cech}) 
which is described below gives a discrete representation of the union of balls.
A classic result in topology called the Nerve Theorem~\cite{hatcher01} 
states that the homology of
$\bigcup_{y\in U}B_\epsilon(y)$
is identical to the homology of the corresponding \v Cech complex.

We now describe a simplicial complex and its homology. 
A \emph{simplicial complex} is a hereditary set system $\simpcomp$
over a vertex set $V$, i.e. $\sigma\subset \sigma'\in \simpcomp$
implies that $\sigma\in \simpcomp$.  The \emph{dimension} of a simplex
$\sigma$ is $|\sigma|-1$; singletons are $0$-simplices or vertices,
pairs in $\cal{K}$ are $1$-simplices or edges, triples are
$2$-simplices or triangles, etc.  A $p$-chain is a formal sum of
$p$-simplices.  The coefficients are taken in $\Z_2$, the integers mod
$2$.%
\footnote{In general, homology may be defined over any ring, but we
stick with $\Z_2$ for ease of exposition and computation.}  Thus,
chains may be viewed as subsets of simplices and addition (mod 2)
as symmetric
difference of sets.  Addition of chains forms an abelian group called
the \emph{chain group} $C_p$ with $0$ denoting the empty chain.

A $p$-simplex $\sigma =\{v_0,\ldots,v_p\}$ has $p+1$ simplices of
dimension $p-1$ on its boundary, denoted $\sigma_i =
\sigma\setminus\{v_i\}$.  The \emph{boundary} of a simplex is $\bdy_p
\sigma = \sum_{i=0}^p \sigma_i$.  The \emph{boundary operator}
$\bdy_p:C_p \to C_{p-1}$ is the natural extension of the boundary of a
simplex to the boundary of a chain:
$\bdy_p c = \sum_{\sigma\in c}\bdy_p \sigma$.

The kernel and image of the boundary operator are two important subgroups of the chain group:
\emph{the cycle group}: $Z_p = \ker \bdy_p = \{z\in C_p : \bdy_p z = 0\}$, and
\emph{the boundary group}: $B_p = \im \bdy_p = \{\bdy_{p+1} c : c\in C_{p+1}\}$.
The \emph{cycles} $Z_p$ are those chains that have boundary $0$. 
The \emph{boundary cycles} $B_p$ are those $p$-chains that are the boundary of some $p+1$-chain.
It is easy to check that $\bdy_{p-1}\bdy_p c = 0$ and thus $B_p\subset Z_p\subset C_p$.
See Figure~\ref{fig:chain_complexes}.

\begin{figure*}[ht] 
  \centering
  \includegraphics[width=.5\textwidth]{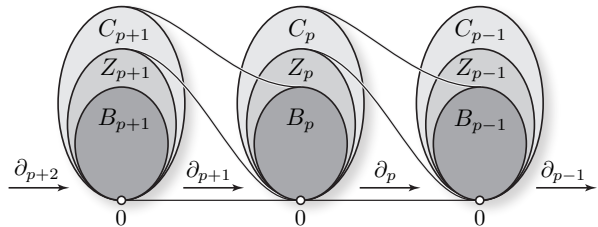}
  \caption{\label{fig:chain_complexes}
    Relationship between chains $C_p$, 
    cycles $Z_p = {\rm ker}\ \partial_p$ and boundaries
    $B_p= {\rm im}\ \partial_{p+1}$.
    The chains $C_p$ are just collections of simplices.
    The chains in $Z_p$ are the cycles.
    The cycles in $B_p$ are the cycles that happen to be boundaries of chains in $C_{p+1}$.
  }
\end{figure*} 

Two cycles $z_1,z_2\in Z_p$ are \emph{homologous} if $z_1-z_2\in B_p$,
i.e.\ their difference is the boundary of a $p+1$-chain.  The $p$th
homology group $H_p$ is defined as the quotient group $Z_p / B_p$.
That is, the homology group is a collection of equivalence classes of
cycles.  The first homology group $H_0$ corresponds to connected
components (clusters).  The next homology group $H_1$ corresponds to
cycles (or loops).  Higher order homology groups correspond to
equivalence classes of higher dimensional cycles.\footnote{
Intuitively, boundary cycles are ``filled in''
cycles and two cycles are homologous if one cycle can be deformed
into the other cycle.} The homology of $\mathcal{K}$ is the collection $\mathcal{H}$ of all its homology groups. 


\begin{figure}[ht] 
  \centering
  \includegraphics[width=.4\textwidth]{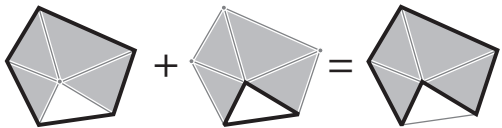}
  \caption{\label{fig:homologous_cycles}
    The sum of two $1$-cycles is another $1$-cycle. Here the cycles are homologous because their sum (in $\Z_2$)
    is the boundary of a $2$-chain of triangles.
  }
\end{figure} 


The \v Cech complex is a specific simplicial complex defined as follows. 
Fix some $\epsilon >0$ and a set of points $S\subset \R^D$.
The \emph{\v Cech complex} consists of all simplices
$\sigma$ such that
$\bigcap_{x\in\sigma}B_\epsilon(x) \neq \emptyset$
where
$B_\epsilon(x)$ is a ball of radius $\epsilon$ centered at $x$.
See Figure \ref{fig::cech}.


\begin{figure}[!ht] 
  \centering
  \includegraphics[width=.4\textwidth]{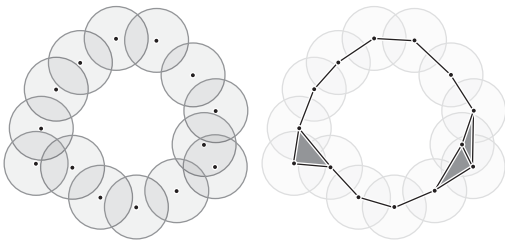}
  \caption{\label{fig:cech_complex}
    A union of balls and its corresponding \v Cech complex.
  }\label{fig::cech}
\end{figure} 
Since the coefficient ring is a field, the computations may be
completely described by linear algebra. The groups $C_p$, $Z_p$,
$B_p$, and $H_p$ are vector spaces and the boundary operators are
linear maps.  It is possible to efficiently compute the homology
groups of a simplicial complex in time polynomial in the size of the
complex. The algorithm only involves row reduction on the matrix
representations of $\bdy_p$.

\vspace{-.3cm}
\section{Techniques}
\vspace{-.3cm}
\label{sec::prelim}
\subsection{Techniques for lower bounds}
The {\em total variation distance} between two measures $P$ and $Q$ is defined by
${\sf TV}(P,Q) = \sup_A |P(A) - Q(A)|$
where the supremum is over all measurable sets.
It can be shown that
${\sf TV}(P,Q) = P(G) - Q(G) = 1 - \int \min (P,Q)$
where $G = \{y:\ p(y) \geq q(y)\}$
and $p$ and $q$ are the densities of $P$ and $Q$
with respect to any measure $\mu$ that dominates both $P$ and $Q$.

We shall make repeated use of Le Cam's lemma 
which we now state (see, e.g., \cite{bin_yu}).

\begin{lemma}[{\bf Le Cam}]
\label{lemma::lecam}
Let ${\cal Q}$
be a set of distributions.
Let $\theta(Q)$ take values in a metric space with metric $\rho$.
Let $Q_1,Q_2\in {\cal Q}$ be any pair of distributions in ${\cal Q}$.
Let $Y_1,\ldots, Y_n$ be drawn iid from some $Q\in {\cal Q}$
and denote the corresponding product measure by
$Q^n$.
Then
\begin{eqnarray*}
\inf_{\hat\theta}
\sup_{Q\in {\cal Q}} 
& & \mathbb{E}_{Q^n}\Bigl[\rho(\hat\theta,\theta(Q))\Bigr]  \geq \\
& & \frac{1}{8} \rho(\theta(Q_1),\theta(Q_2)) \  (1- {\sf TV}(Q_1,Q_2))^{2n}
\end{eqnarray*}
where the infimum is over all estimators.
\end{lemma}
Le Cam's lemma makes precise the intuition that if there are
\emph{distinct} members of the class $\mathcal{Q}$ for which the data generating 
distributions are close 
then the statistical problem
is hard given a small sample.

When we apply Le Cam's lemma in this paper,
$Q_1$ and $Q_2$ will be associated with two different manifolds
$M_1$ and $M_2$.
We will take $\theta(Q)$ to be the homology of the manifold and
$\rho(\theta(Q_1),\theta(Q_2)) =1$ if the homologies are the different and
$\rho(\theta(Q_1),\theta(Q_2)) =0$ if the homologies are the same.
The subtlety of establishing \emph{tight} lower bounds 
boils down to the task of finding a set of distributions in the class 
$\mathcal{Q}$ for which the homology of the underlying submanifolds
are distinct but whose empirical distributions are hard to distinguish
from a small number of samples.

We will use two representative manifolds $M_1$ and $M_2$ in the application of
LeCam's lemma which we describe here. See Figure \ref{fig::two-manifolds}.
The manifold $M_1$ is a pair of $1-\tau$ $d$-balls (shown in blue) embedded $2\tau$ apart in $\mathbb{R}^D$ joined smoothly
at their ends (shown in red). The manifold $M_2$ is a pair of $d$-annuli (shown in blue) embedded $2\tau$ apart with outer radius $1-\tau$ and inner radius $4\tau$, smoothly joined at both the inner and outer ends (shown in red).
It is clear from the construction that both these manifolds are d-dimensional compact, have no boundary and have
condition number $\tau$. It is also the case that ${\cal H}(M_1)\neq {\cal H}(M_2)$.


\begin{figure*}[bht]
\begin{center}
\includegraphics[scale=.35]{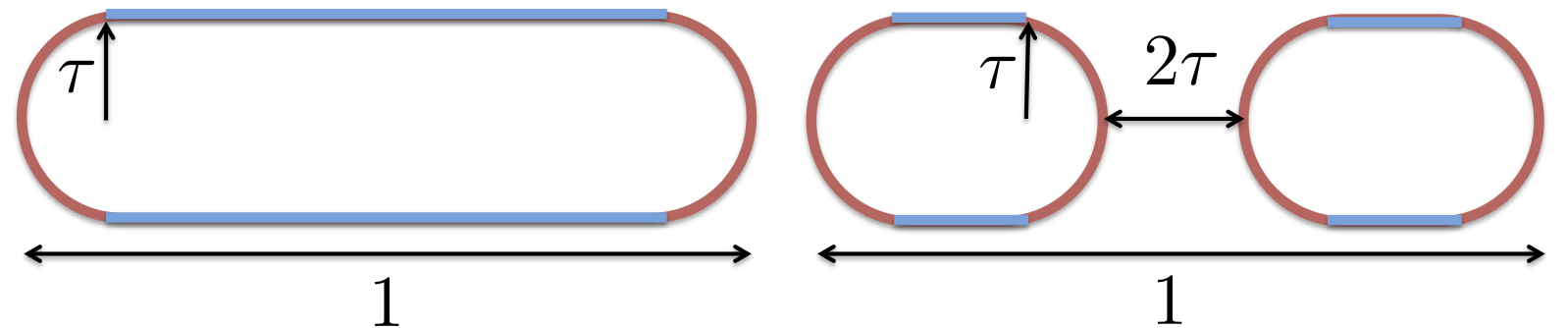}
\end{center}
\caption{The two manifolds $M_1$ and $M_2$, with d = 1, D = 2}
\label{fig::two-manifolds}
\end{figure*}
If there exist two manifolds
$M_1$ and $M_2$ with corresponding distributions
$Q_1$ and $Q_2$ in ${\cal Q}$
such that (i) ${\cal H}(M_1) \neq {\cal H}(M_2)$
and (ii) $Q_1=Q_2$ then
we say that the model ${\cal Q}$ is
{\em non-identifiable}.
In this case, recovering the homology is impossible and
we write $R_n=1$ and $n(\epsilon)=\infty$.
\vspace{-.1cm}
\subsection{Techniques for upper bounds}
\vspace{-.1cm}
To establish an upper bound we need to construct an estimator that achieves the upper bound. 
In the noiseless and tubular noise cases the samples are in a thin region around the manifold and our estimator is constructed from a union of balls (of a carefully chosen radius) around the sample points.

In the case of clutter noise and additive Gaussian noise samples are concentrated around the manifold but a few samples may be quite far away from the manifold. 
In these cases our upper bounds  are obtained by analyzing the performance of the Algorithm \ref{clean} ({\sf CLEAN}) with a carefully specified threshold and radius, which is used to remove points in regions of low density far away from the manifold. Our estimator is then constructed from a union of balls around the remaining points.
\begin{algorithm}
\caption{{\sf CLEAN}}
\label{clean}
\begin{itemize}
\item IN: $(X_i)_{i=1}^n$, threshold $t$, radius $r$
\end{itemize}
\begin{packed_enum}
\item Construct a graph $G_r$ with nodes $\{ X_i\}_{i=1}^n$. 
Include edge $(X_i,X_j)$, if $||X_i - X_j|| \leq r$.
\item Mark all vertices with degree $d_i \leq (n-1)t$.
\end{packed_enum}
\begin{itemize}
\item OUT: All unmarked vertices
\end{itemize}
\end{algorithm}
In the case of additive noise with general known distribution the samples are not expected to be concentrated around the manifold. We will first use deconvolution to estimate 
a deconvolved measure $\hat{P}_{n}$ which we will show is densely concentrated
in a thin region around the manifold. We will then draw samples from this measure, clean them and construct 
a union of balls of appropriate radius around the remaining samples, and show that this set has the right homology with high probability.

We now briefly review statistical deconvolution. We refer the interested reader to the work of Fan \cite{fan_deconvolution} for more details and to \cite{koltchinskii} for an application related to ours. The procedure is similar to kernel density estimation with a kernel modified
to account for the additive noise.
For symmetric noise distributions $\Phi$, we consider two kernels ${\cal K}$ and ${\Psi}$ 
such that ${\cal K} \star \Phi = \Psi$, where $\star$ denotes convolution. 
The deconvolution estimator is $\hat{P}_n(A) = 1/n \sum_{i = 1}^n {\cal K}(Y_i - A)$.
It is easy to verify that $E\hat{P}_n = P \star \Psi$ similar to regular kernel density estimation with the kernel $\Psi$.
In the noiseless case we can even take ${\cal K} = \Psi = \delta_0$ (a Dirac at 0) and get back the empirical distribution of the sample. More generally, we will be interested in 
$\Psi$ that satisfies 
$\Psi \{ x: |x| \geq \epsilon\} \leq \gamma$.
for $\epsilon$ and $\gamma$ that we will later specify.


In each of the above cases our final estimator is constructed from a union of balls around appropriate points,
and our theorems will show that these have the correct homology with high
probability. To compute the homology one would construct the corresponding
\v Cech complex and compute its ``boundary matrices'' (as described in Section \ref{sec::homology}). Recovering
the homology from these matrices consists of linear algebraic manipulation. 
There are several fast algorithms to compute the homology (either exactly \cite{plex} or
approximately \cite{witness}) of the \v Cech complexes from large point sets in high dimensions.

\vspace{-.4cm}
\section{Minimax Rates}\label{sec:rates}
\vspace{-.3cm}
We now derive the minimax rates for homology estimation under the four noise models described in section \ref{sec::model}.
There are three quantities of interest: the minimax risk $R_n$,  the resolution $\tau_n$
and the sample complexity $n(\epsilon)$.
We write $R_n \asymp a_n$ (similarly for $\tau_n \asymp a_n$)
if there are positive constants $c$ and $C$ such that
$c \leq R_n/a_n \leq C$ for all large $n$.
Similarly, we write
$n(\epsilon) \asymp a(\epsilon)$
if there are positive constants $c$ and $C$ such that
$c \leq n(\epsilon)/a(\epsilon) \leq C$ for all small $\epsilon$.
Our analysis will show that the rates (as a function of $n$)
are typically polynomial for the resolution and exponential for the risk.
We will often match upper and lower bounds on sample complexity and resolution
only up to logarithmic factors, and correspondingly those on the risk upto polynomial factors. In this
case we will use the notation $R_n \asymp^\ast a_n$, $\tau_n  \asymp^\ast a_n$, and $n(\epsilon) \asymp^\ast a(\epsilon)$.

It is worth emphasizing at this point that despite the fact that we use two specific manifolds
in the application of Le Cam's lemma, the resulting lower bound holds for \emph{all}
manifolds in $\cal M$ and \emph{all} distributions in $\cal Q$. 
Le Cam's lemma allows one to get a
lower bound that holds for \emph{any} estimator
 by using \emph{two} carefully chosen
distributions in $\cal Q$. The upper bounds are from specific estimators
and they establish an
upper bound on the number of samples to estimate the homology of any
manifold in our class.

\vspace{-.2cm}
\subsection{Noiseless Case}
\vspace{-.2cm}
\label{sec::noiseless}
\begin{theorem}
\label{thm:noiseless}
For all $\tau \leq \tau_0(a,d)$, in the noiseless case the minimax rate,
$R_n \asymp^\ast e^{-n \tau^d}$, where $\tau_0(a,d)$ is a constant which depends
on $a$ and $d$. Also,
$n(\epsilon) \asymp^\ast \tau^{-d}\log(1/\epsilon)$ and
$\tau_n \asymp^\ast ( (1/n)\log(1/\epsilon) )^{1/d}$.
\end{theorem}
We provide proof sketches for the lower and upper bounds
on $R_n$ separately.
\subsubsection*{Lower Bound: Proof Sketch}
To obtain a lower bound on the minimax risk over the class ${\cal Q}(\tau)$
we will consider the two carefully chosen manifolds $M_1$ and $M_2$ described
earlier.

We further need to specify the density on each of the manifolds, and we
choose two densities from ${\cal P}$ so that the data distributions are as similar
as possible while respecting the constraint $p(x) \geq a$. The construction is described
in more detail in the Appendix \ref{app:noiselesslower}, but for now it suffices to notice that
the two densities can be constructed to differ only on the sets $W_1 = M_1 \setminus M_2$ and 
$W_2 = M_2 \setminus M_1$ and can be made as low as $a$ on one of these sets. A straightforward calculation shows that $${\sf TV}(p_1,p_2) \leq a \max( \vol(W_1), \vol(W_2)) \leq  C_d a\tau^d$$ where the constant $C_d$ depends on $d$. Now, we apply Le Cam's lemma to obtain that $$R_n \geq \frac{1}{8} \left( 1 - C_d a \tau^d \right)^n  \geq 
\frac{1}{8}e^{-2 C_d n a \tau^d}$$
for all $\tau \leq \tau_0(a,d)$. $\tau_0(a,d)$ is a constant depending on $a$ and $d$. The lower bound of Theorem \ref{thm:noiseless} follows.

\subsubsection*{Upper Bound: Proof Sketch}
In the noiseless case the samples are densely concentrated around the manifold
and our estimator is constructed from a union of balls of radius $\tau/2$ 
around the sample points. 
The upper bound on the minimax risk 
follows from a straightforward modification of the results of \cite{niyogi2008}. 
For completeness, we reproduce an adaptation of their main homology inference theorem (Theorem 3.1) here.
\begin{lemma} 
\label{lem:nswmain}
[NSW]
Let $0 < \epsilon < \tau$ and
let $U = \bigcup_{i=1}^n B_\epsilon(X_i)$. 
Let $\hat{\cal H} = {\cal H}(U)$.
Let 
$\zeta_1 = \frac{{\sf vol}(M)}{a \cos^d \theta_1 {\sf vol}(B^d_{\epsilon/4})}$, 
$\zeta_2 =  \frac{{\sf vol}(M)}{a \cos^d \theta_2 {\sf vol}(B^d_{\epsilon/8})}$,
$\theta_1 = \sin^{-1} \frac{\epsilon}{8\tau}$ and $\theta_2 = \sin^{-1} \frac{\epsilon}
{16\tau}$.
Then for all 
$n > \zeta_1 \left( \log(\zeta_2) + \log \left( \frac{1}{\delta} \right) \right)$,
$\mathbb{P}(\hat{\cal H} \neq {\cal H}(M)) < \delta$.
\end{lemma} 
By assumption ${\sf vol}(M) \leq C_{D,d}$ for some constant $C_{D,d}$ depending on
$d$ and $D$. 
To obtain a sample complexity bound we simply choose  $\epsilon =  \tau/2$
and this gives us $n(\epsilon) \leq C_1/ (a \tau^d) (C_2 \log (1/ (a\tau^d)) + \log(1/\epsilon))$
which matches the lower bound upto the factor of $\log(1/\tau)$.
Further calculation (see Appendix \ref{app:noiselessup}) 
then shows that as desired $R_n \leq C_1/\tau^d \exp (- C_2 n a \tau^d)$ for appropriate constants
$C_1, C_2$, and $\tau_n \leq C (\frac{\log n \log (1/\epsilon)}{an})^{1/d}$. This establishes 
Theorem \ref{thm:noiseless}.
%

\vspace{-.2cm}
\subsection{Clutter Noise}
\vspace{-.2cm}
\label{sec::clutter}
\begin{theorem}
\label{thm:clutter}
For all $\tau \leq \tau_0(a,d)$, in the clutter noise case,
$R_n \asymp^\ast e^{-n \pi \tau^d}$, where $\tau_0(a,d)$ is a constant which depends
on $a$ and $d$. 
Also,
$n(\epsilon) \asymp^\ast(1/ (\pi \tau^d) \log (1/\epsilon)$ and
$\tau_n \asymp^\ast (1/(n \pi) \log(1/\epsilon) )^{1/d}$.
\end{theorem}
\subsubsection*{Lower Bound: Proof Sketch}
The lower bound for the class ${\cal Q}(\pi, \tau)$ follows via the same construction
as in the noiseless case. In the calculation of the total variation distance  (see Appendix \ref{app:clutterlow})
we have instead
$${\sf TV}(q_1,q_2) \leq \pi a \max( \vol(W_1), \vol(W_2)) \leq  C_d \pi a\tau^d$$where $C_d$ depends on $d$.
As before the lower bound follows from the application of Le Cam's lemma.
\subsubsection*{Upper Bound: Proof Sketch}
\label{sec:clutter.upper}
As a preliminary step we clean the data samples 
to eliminate points that are far
away from, while retaining those close to, the manifold. 
Our analysis shows that Algorithm~\ref{clean} will
 achieve this, with high probability for a carefully chosen threshold and radius. 
We then show that 
taking a union of balls of the appropriate radius
around the remaining points will give us the correct homology,  with high probability.
We give an outline here and defer details to Appendix \ref{app:clutterup}.
\begin{packed_enum}
\item We define two regions $A = {\sf tube}_r({M})$ and
$B = \mathbb{R}^D \setminus {\sf tube}_{2r}({M})$
where $r < \frac{(\sqrt{9} - \sqrt{8}) \tau}{2}$.
\item We then invoke Algorithm {\sf CLEAN} on the data with
threshold 
$t = \left( \frac{v_D s^D (1 - \pi)}{\mathrm{vol(Box)}} +
\frac{\pi av_d r^d \cos^d \theta}{2} \right)$ 
and radius $2r$. Let $I$ be the set of vertices returned.
\item Through careful analysis we show that with high probability
$I$ contains \emph{all} the vertices
from the region $A$ and \emph{none} of the points in region $B$.
\item We further show that the retained points form a thin dense cover
of the manifold $M$, i.e.  $\Bigl\{ {M} \subset \bigcup_{i\in I} B_{2r}(X_i)\Bigr\}$.
\item Using a straightforward corollary of Lemma \ref{lem:nswmain} we show 
that this thin dense cover can be used to recover the homology of $M$ with high probability.
\end{packed_enum}

Formally, in Appendix \ref{app:clutterup} we prove the following lemma,
\begin{lemma}
\label{clutter:main}
If $n > \max(N_1,N_2)$, and $r < (\sqrt{9} - \sqrt{8}) \frac{\tau}{2}$ where
$N_1 = 4\kappa \log (\kappa)$
\begin{eqnarray*}
\mathrm{with}~~\kappa & = & \max \left(1 + \frac{200}{3\zeta} \log \left( \frac{2}{\delta} \right) ,4 \right) \\
\mathrm{and}~N_2 & = & \frac{1}{\zeta} 
\left(\log \left(\frac{\mathrm{vol}({M})}{\cos^d(\theta)v_d r^d}\right) + 
\log \left(\frac{2}{\delta}\right) \right)
\end{eqnarray*}
where $\zeta = \pi av_d r^d \cos^d (\theta)$ and $\theta = \sin ^{-1} (r/2\tau)$, 
then after cleaning the points $\{X_i:\ i\in I\}$ are
all in ${\sf tube}_{2r}({M})$ and
are $2r$ dense in ${M}$.
Let
$U = \bigcup_{i\in I}B_w(X_i)$ with
$w=r + \frac{\tau}{2}$ and let
$\hat{\cal H} = {\cal H}(U)$.
We have that $\hat{\cal H} = {\cal H}(\cal M)$
with probability at least $1-\delta$.
\end{lemma}
Taking $r = (\sqrt{9} - \sqrt{8})\tau/4$, we obtain the sample complexity bound, $n(\epsilon) \leq 
\frac{C_1}{\pi \tau^d} (\log \frac{C_2}{\tau^d} + \log(C_3/\epsilon))$. Given this sample complexity upper bound, the upper bounds on minimax risk and resolution follow identical arguments to the noiseless case (Appendix \ref{app:noiselessup}).

\vspace{-.2cm}
\subsection{Tubular Noise}
\vspace{-.2cm}
Under this noise model we get samples 
uniformly from a tubular region of width $\sigma$ around the manifold.
This model highlights an important phenomenon in high-dimensions. Although, 
we receive samples \emph{uniformly} 
from a full $D$ dimensional shape these samples
concentrate tightly around a $d$ dimensional manifold. We show 
that with some care we can still reconstruct the homology at 
a rate independent of $D$.
\begin{theorem}
Under the tubular noise model we establish the following cases.
\begin{enum}
\item If $\sigma \geq 2 \tau$ then 
the model is non-identifiable and hence,
$R_n=1$ and $n(\epsilon)>\infty$.
\item If $\sigma \leq C_0 \tau$, with $C_0$ small and $\tau \leq \tau_0(a,d)$,
then
$R_n \asymp^\ast e^{-n \tau^d}$, where $\tau_0(a,d)$ is a constant which depends
on $a$ and $d$. 
Also,
$n(\epsilon) \asymp^\ast 1/\tau^d$ and
$\tau_n \asymp^\ast  (\frac{1}{n}\log(1/\epsilon))^{1/d}$.
\end{enum}
\end{theorem}
\begin{remark}
The case when $\sigma$ is very close to $\tau$ is significantly more involved since
it involves the \emph{exact} calculation of the volume of the tubular region and establishing
tight upper and lower bounds here is an open problem we are attempting to address in current work.
\end{remark}
\subsubsection*{Lower bound: Proof Sketch}
\begin{packed_enum}
\item 
When $d<D$ and $\sigma \geq 2 \tau$
the two manifolds $M_1$ and $M_2$ that we have considered thus far
are still identifiable because even when $\sigma \geq \tau$ $M_2$
has a ``dimple'' along the co-dimensions that $M_1$ does
not.
To show that the class ${\cal Q}$ is still not identifiable we require a different construction.
Consider the manifolds $M_1$ and $M_2$ with two points
placed above and below the manifold at a distance $2\tau$ above their
centers along each of the co-dimensions. Denote these new manifolds
$M_1^\prime$ and $M_2^\prime$.  It is clear that 
${\cal H}(M_1^\prime) \neq {\cal H}(M_2^\prime)$, however $Q_1^\prime =
Q_2^\prime$ since the extra points hide the ``dimple'' and the two
manifolds cannot be distinguished.


\item When $d < D$, and $\sigma \leq C_0 \tau$ we return to our old constructions
of $M_1$ and $M_2$. There is however an important difference in that 
the two manifolds differ on full $D$-dimensional sets, and one might suspect that
${\sc TV}(q_1,q_2) = O(\tau^D)$ or perhaps $O(\sigma^{D-d} \tau^d)$. As we show
in Appendix \ref{app:tub} however, ${\sc TV}(q_1,q_2)$ is still $O(\tau^d)$,
and we recover an identical lower bound to the noiseless case.

\end{packed_enum}

\subsubsection*{Upper bound: Proof Sketch}
We are interested in case when $\sigma \leq C_0\tau$ (in particular
$\sigma < \tau/24$ will suffice). Our proof will involve two main steps which we sketch here.
\begin{packed_enum}
\item We first show that if we consider balls of sufficiently
large radius $\epsilon$  (compared to $\sigma$) then the probability mass 
in these balls is $O(\epsilon^d)$. This is a manifestation of the phenomenon alluded
to earlier: inside large enough balls the mass is concentrated around the lower dimensional
manifold. Precisely, define $k_\epsilon = \inf_{p \in {M}} Q(B_\epsilon(p))$.  In
Lemma \ref{tubular_cover_lemma} in the Appendix, we show that, if
$\epsilon \gg \sigma$ is large, $k_\epsilon$ is of order
  $\Omega(\epsilon^d)$.

\item There is however a disadvantage to considering balls that are too large. The homology of the union of balls around the samples may no longer have the right homology. Using tools from 
NSW, we show in the Appendix 
that we can balance these two considerations for manifolds with high condition number, i.e. provided $\sigma < \tau/24$, we can choose
balls that are both large relative to $\sigma$ and whose union still has the correct homology.

\end{packed_enum}
We will prove the following main lemma in the Appendix.
\begin{lemma}
Let $N_{\epsilon}$ be the $\epsilon$-covering number of the submanifold ${M}$. 
Let $U = \bigcup_{i=1}^n B_{\epsilon + \tau/2}(X_i)$.
Let $\hat{\cal H} = {\cal H}(U)$.
Then if $n > \frac{1}{k_{\epsilon}} \left( \log(N_{\epsilon}) + \log(1/\delta) \right)$,
$\mathbb{P}(\hat{\cal H} \neq {\cal H}(M)) < \delta$ as long as
$\sigma \leq \epsilon/2$ and $\epsilon < \frac{(\sqrt{9} - \sqrt{8})\tau}{2}$.
\end{lemma}
Notice, that we require $\sigma < \frac{(\sqrt{9} - \sqrt{8})\tau}{4}$ which is satisfied if $\sigma < \tau/24$ (for instance).
To obtain the upper bound set $\epsilon = 2\sigma$, and observe that $N_{\epsilon} = O(1/\epsilon^d) = O(1/\tau^d)$ and $k_{\epsilon} = O(\epsilon^d) = O(\tau^d)$. This gives us that if $n \geq \frac{C_1}{\tau^d} ( \log (\frac{C_2}{\tau^d}) + \log (\frac{1}{\delta}) )$ we recover the
right homology with probability at least $1-\delta$. The upper bound on minimax risk and resolution follows from similar arguments to those made previously.

\vspace{-.2cm}
\subsection{Additive Noise}
\vspace{-.2cm}
\label{sec::additive}
For additive noise we consider two cases. 
In the first case, 
we derive the minimax rates for additive \emph{Gaussian} noise 
under the somewhat restrictive
assumption that $C \sqrt{D} \sigma < \tau$. 
This problem is related of the problem of separating mixtures of
Gaussians (which corresponds to the case where the manifold is a
collection of points and $2\tau$ is the distance between the closest
pair). 
In this case have the following theorem.
\begin{theorem}\label{thm:additive}
For all $\tau \leq \tau_0(a,d)$ and $8 \sqrt{D} \sigma <  \tau$,
$R_n \asymp^\ast e^{- n \tau^d}$, where $\tau_0(a,d)$ is a constant which depends
on $a$ and $d$. 
Also,
$n(\epsilon) \asymp^\ast (1/\tau^d)\log(1/\epsilon)$
and
$\tau_n \asymp^\ast \left( (1/n)\log (1/\epsilon)\right)^{1/d}$.
\end{theorem}
As in the clutter noise case we need to first clean the data and then
take a union of balls around the points which survive. We analyze this
procedure in the Appendix.

\vspace{-.2cm}
\subsubsection{Deconvolution}
\label{sec:decon}
\vspace{-.2cm}
%
Here we consider more general \emph{known} noise distributions but
work over the class of distributions ${\cal Q}(\Phi)$ over manifolds with $\tau$ fixed.
We first use deconvolution to estimate 
a deconvolved measure $\hat{P}_{n}$ which is concentrated
around the manifold. We then draw samples from this measure, clean them and construct 
a union of balls $H$ around these samples, and show that $H$ has the right homology with
high probability.
The class of noise distributions we will consider satisfy the following assumption on its density.
\begin{assumption}
\label{ass:fourier}
Denote $ \rho(R) = \inf_{|t|_\infty \leq R} |\Phi^\star(t)|$, where 
$R > 0, |t|_\infty = \max_{1 \leq j \leq m} |t_j| $ and $\Phi^\star(t)$ is the Fourier transform of the
symmetric
noise density $\Phi$. We assume $\rho(R) > 0$.
\end{assumption}
This is a standard assumption in the literature on deconvolution (see \cite{fan_deconvolution,koltchinskii}), since as  described deconvolution requires us to divide by the Fourier transform of the noise which needs to be bounded away from 0 for the procedure to be well behaved. The assumption is satisfied by a variety of noise distributions including Gaussian noise.
Our main result says that for this broad class of noise distributions the deconvolution procedure described above will achieve an optimal rate of convergence.
\begin{theorem}
\label{thm:decon}
In the additive noise case with $\tau$ fixed for $\Phi$ satisfying Assumption \ref{ass:fourier}.
$R_n \asymp e^{-n}$.
Hence,
$n(\epsilon) \asymp \log (1/\epsilon)$.
\end{theorem}
{\bf Lower Bound: Proof Sketch}To obtain the lower bound one can consider the same construction from the previous subsection with additive Gaussian noise. If $\tau$ is taken to be fixed we obtain the desired bound.

{\bf Upper Bound: Proof Sketch}
Our proof of the upper bound follows similar lines to that of Koltchinskii \cite{koltchinskii}. We deviate in two significant aspects. Koltchinskii only assumes an upper bound on the density, which he shows is sufficient to estimate weak geometric characteristics like the dimension of the manifold. To show that we can accurately reconstruct its homology we require both an upper and lower bound and our methods are quite different. Koltchinskii uses an epsilon net of the \emph{entire} compact set containing the manifold critically in his construction and his procedure is thus not implementable/practical. Our algorithm instead draws a small number of samples from the deconvolved measure and uses those to estimate the homology resulting in a practical procedure.
We prove the following upper bound in the Appendix.
\begin{lemma}
\label{thm:decon}
Given $n$ samples from ${\cal Q}(\Phi)$ with $\Phi$ satisfying Assumption \ref{ass:fourier},
there exist $C_1,C_2,c_1 > 0$ such that
$P( {\cal H}(H) \neq {\cal H}(M) ) \leq C_1 e^{-c_1n} $, 
where $H$ is a union of balls of radius $\frac{5\epsilon + \tau}{2}$ centered around $m \geq C_2 n $ samples drawn from the deconvolved measure $\hat{P}_{n}$ with a kernel $\Psi$ with parameters $\gamma, \epsilon$ (specified in the proof). The samples are cleaned using the deconvolved measure by considering balls of radius $4\epsilon$ at a threshold $2 \gamma$.
\end{lemma}
\begin{remark}
The cleaning procedure we use here is different from the Algorithm {\sc CLEAN}. We remove
points around which a ball of appropriate radius has low probability mass under the deconvolved measure. This is equivalent to using the deconvolved measure in place of the k-NN density estimate implicitly constructed by the {\sc CLEAN} procedure.
\end{remark}
Simple calculations show that this lemma together with the lower bound give the exponential minimax rate described in Theorem \ref{thm:decon}.
\vspace{-.3cm}
\section{Conclusion}
\vspace{-.3cm}
We have given the first minimax bounds
for homology inference. These bounds give insight into
the intrinsic difficulty of the problem
under various assumptions. Our bounds show that it is often 
possible to estimate the homology of a manifold at
fast rates independent of the ambient dimension.

Actual implementation of homology inference
has become tractable thanks to advances
in computational topology.
However, as our proofs reveal, recovering the homology
requires the careful selection of several tuning parameters.
In current work, we are developing methods
for choosing these parameters in a statistically sound, data-driven way.
\newpage 

\bibliography{Homology}

\begin{thebibliography}{10}

\bibitem{adler}
Robert~J. Adler, Omer Bobrowski, Matthew~S. Borman, Eliran Subag, and Shmuel
  Weinberger.
\newblock Persistent homology for random fields and complexes.
\newblock In James~O. Berger, Tony Cai, and Iain~M. Johnstone, editors, {\em
  Borrowing Strength: Theory Powering Applications - A Festschrift for Lawrence
  D. Brown}, pages 124--143. Institute of Mathematical Statistics, 2010.

\bibitem{chazal2011}
Frederic Chazal, David Cohen-Steiner, and Quentin Merigot.
\newblock Geometric inference for probability measures.
\newblock {\em Foundations of Computational Mathematics}, 2011.
\newblock to appear.

\bibitem{bubenik09}
Moo~K. Chung, Peter Bubenik, and Peter~T. Kim.
\newblock Persistence diagrams of cortical surface data.
\newblock In {\em Proceedings of the 21st International Conference on
  Information Processing in Medical Imaging}, IPMI '09, pages 386--397, Berlin,
  Heidelberg, 2009. Springer-Verlag.

\bibitem{plex}
Vin de~Silva.
\newblock {\em PLEX: Simplicial complexes in MATLAB}.

\bibitem{witness}
Vin de~Silva and Gunnar Carlsson.
\newblock {Topological estimation using witness complexes}.
\newblock In M.~Alexa and S.~Rusinkiewicz, editors, {\em Eurographics Symposium
  on Point-Based Graphics}. The Eurographics Association, 2004.

\bibitem{silva07}
Vin de~Silva and Robert Ghrist.
\newblock {Coverage in sensor networks via persistent homology}.
\newblock {\em Algebraic \& Geometric Topology}, 7:339--358, 2007.

\bibitem{dequeant08}
Mary-Lee Dequeant, Sebastian Ahnert, Herbert Edelsbrunner, Thomas M.~A. Fink,
  Earl~F. Glynn, Gaye Hattem, Andrzej Kudlicki, Yuriy Mileyko, Jason Morton,
  Arcady~R. Mushegian, Lior Pachter, Maga Rowicka, Anne Shiu, Bernd Sturmfels,
  and Olivier Pourquié.
\newblock Comparison of pattern detection methods in microarray time series of
  the segmentation clock.
\newblock {\em PLoS ONE}, 3(8):e2856, 08 2008.

\bibitem{edelsbrunner09}
H.~Edelsbrunner and J.H. Harer.
\newblock {\em Computational topology}.
\newblock American mathematical society, 2009.

\bibitem{fan_deconvolution}
Jianqing Fan.
\newblock On the optimal rates of convergence for nonparametric deconvolution
  problems.
\newblock {\em Ann. Statist.}, 19(3):1257--1272, 1991.

\bibitem{gamble10}
Jennifer Gamble and Giseon Heo.
\newblock Exploring uses of persistent homology for statistical analysis of
  landmark-based shape data.
\newblock {\em Journal of Multivariate Analysis}, 101(9):2184--2199, October
  2010.

\bibitem{hatcher01}
Allen Hatcher.
\newblock {\em Algebraic Topology}.
\newblock Cambridge University Press, 2002.

\bibitem{hein06manifold}
Matthias Hein and Markus Maier.
\newblock Manifold denoising.
\newblock In Bernhard Sch{\"o}lkopf, John~C. Platt, and Thomas Hoffman,
  editors, {\em NIPS}, pages 561--568. MIT Press, 2006.

\bibitem{Kahle20091658}
Matthew Kahle.
\newblock Topology of random clique complexes.
\newblock {\em Discrete Mathematics}, 309(6):1658 -- 1671, 2009.

\bibitem{kasson07}
P.~M. Kasson, A.~Zomorodian, S.~Park, N.~Singhal, L.~J. Guibas, and V.~S.
  Pande.
\newblock Persistent voids: a new structural metric for membrane fusion.
\newblock {\em Bioinformatics}, 23(14):1753--1759, 2007.

\bibitem{koltchinskii}
V.~I. Koltchinskii.
\newblock Empirical geometry of multivariate data: a deconvolution approach.
\newblock {\em Ann. Statist.}, 28(2):591--629, 2000.

\bibitem{niyogi2008}
Partha Niyogi, Stephen Smale, and Shmuel Weinberger.
\newblock Finding the homology of submanifolds with high confidence from random
  samples.
\newblock {\em Discrete {\&} Computational Geometry}, 39(1-3):419--441, 2008.

\bibitem{niyogi2010}
Partha Niyogi, Stephen Smale, and Shmuel Weinberger.
\newblock A topological view of unsupervised learning and clustering.
\newblock {\em SIAM J. Comput.}, 40(3):646--663, 2011.

\bibitem{pascucci01}
Valerio Pascucci, Xavier Tricoche, Hans Hagen, and Julien Tierny.
\newblock {\em Topological methods in Data Analysis and Visualization: Theory,
  Algorithms and Applications}.
\newblock Springer, 2001.

\bibitem{sacan07}
Ahmet Sacan, Ozgur Ozturk, Hakan Ferhatosmanoglu, and Yusu Wang.
\newblock Lfm-pro: a tool for detecting significant local structural sites in
  proteins.
\newblock {\em Bioinformatics}, 23:709--716, February 2007.

\bibitem{ghrist07}
Vin~De Silva and Robert Ghrist.
\newblock Homological sensor networks.
\newblock {\em Notices of the American Mathematical Society}, 54:2007, 2007.

\bibitem{singh08}
Gurjeet Singh, Facundo Memoli, Tigran Ishkhanov, Guillermo Sapiro, Gunnar
  Carlsson, and Dario~L. Ringach.
\newblock {Topological analysis of population activity in visual cortex}.
\newblock {\em J. Vis.}, 8(8):1--18, 6 2008.

\bibitem{bin_yu}
Bin Yu.
\newblock Assouad, {F}ano, and {L}e {C}am.
\newblock In D.~Pollard, E.~Torgersen, and G.~Yang, editors, {\em Festschrift
  for Lucien Le Cam}, pages 423--435. Springer, 1997.

\bibitem{afra05}
Afra Zomorodian.
\newblock {\em Topology for Computing}.
\newblock Cambridge University Press, 2005.

\end{thebibliography}
\bibliographystyle{plain}

\newpage

\appendix
\section{Appendix -- Supplementary Material}
\subsection{Key technical lemmas from \cite{niyogi2008}}
We will need two technical lemmas, which follow from \cite{niyogi2008}.

\begin{lemma}[{\bf Ball volume lemma}, Lemma 5.3 in \cite{niyogi2008}]
\label{lem:nsw_vol} 
Let $p \in M$. Now consider $A = M \cap B_\epsilon(p)$. Then 
$vol(A) \geq (\cos(\theta))^d vol(B^d_\epsilon(p))$ where $B_\epsilon^d(p)$ is the
a $d$-dimensional ball in the tangent space at $p$, $\theta = \sin^{-1} \frac{\epsilon}{2\tau}$.
\end{lemma}

Next, consider a collection of balls $\{ B_r(p_i) \}_{i=1,\ldots,n}$ centered around points $p_i$ on the manifold and such that $M \subset \cup_{i=1}^l B_r(p_i)$.

\begin{lemma}[{\bf Sampling lemma}, Lemma 5.1 in \cite{niyogi2008}]
\label{sample}
Let $A_i = B_{r}(p_i)$ be a collection of sets such that $\cup_{i=1}^l A_i$ forms a minimal cover
of $M$. If $Q(A_i) \geq \alpha$, and 
$$n > \frac{1}{\alpha} \left(\log l + \log \left(\frac{2}{\delta}\right) \right)$$
then w.p. at least $1 - \delta/2$, each $A_i$ contains at least one sample point, and $M \subset \cup_{i=1}^n B_{2r}(x_i)$. Further we have that $l \leq \frac{\mathrm{vol}(M)}{\cos^d(\theta)v_d r^d}$.
\end{lemma}

\subsubsection{Proofs for the noiseless case}
{\bf Lower bound}
\label{app:noiselesslower}
Here we describe the densities on the two manifolds $M_1$ and $M_2$. There are two sets
of interest to us: $W_1 = M_1 \setminus M_2$ which corresponds to the two ``holes'' of radius
$4\tau$  in the annulus, and $W_2 = M_2 \setminus M_1$ which corresponds to the 
$d$-dimensional piece added to smoothly join the inner pieces of the two annuli in $M_2$.

By construction, $\vol(W_1) = 2 v_d (4\tau)^d$ where $v_d$ is the volume of the unit $d$-ball.
$\vol(W_2)$ is somewhat tricky to calculate exactly due to the curvature of $W_2$ but it is easy to
see that $\vol(W_2)$ is also $O(\tau^d)$ with the constant depending on $d$. 

One of the densities is constructed in the following way, on the set of larger volume (between $W_1$ and $W_2$) we set $p(x) = a$, and evenly distribute the rest of the mass over the remaining portion of the manifold
(we are guaranteed that the mass on the rest of the manifold is at least $a$ since otherwise the constraint $p(x) \geq a$ can never be satisfied).

The other density is constructed to be equal (to the first density) outside the set on which the two manifolds differ. The remaining mass is spread evenly on the set where they do differ. We are again guaranteed that $p(x) \geq a$ by construction.

Let us now calculate the TV between these two densities. This is just the integral of the difference
of the densities over the set where one of the densities is larger. Since the two densities are
equal outside $W_1 \cup W_2$ and disjoint over $W_1 \cup W_2$ it is clear that
$${\sc TV}(p_1,p_2) = a \max(\vol(W_1),\vol(W_2) \leq O(a \tau^d) $$
with the constant depending on $d$. The lower bound follows from the calculations in the main paper.

{\bf Upper bound}
\label{app:noiselessup}
The NSW lemma tells us that for $n > \zeta_1 \left( \log(\zeta_2) + \log \left( \frac{1}{\delta} \right) \right)$, with $\zeta_1 = \frac{{\sf vol}(M)}{a \cos^d \theta_1 {\sf vol}(B^d_{\epsilon/4})}$, 
$\zeta_2 =  \frac{{\sf vol}(M)}{\cos^d \theta_2 {\sf vol}(B^d_{\epsilon/8})}$,
$\theta_1 = \sin^{-1} \frac{\epsilon}{8\tau}$ and $\theta_2 = \sin^{-1} \frac{\epsilon}
{16\tau}$, we have $\mathbb{P}(\hat{\cal H} \neq {\cal H}(M)) < \delta$.

By assumption, we have $\vol(M) \leq C$. We further take $\epsilon = \tau/2$. It is clear that
in $\zeta_1$ and $\zeta_2$ all terms except the ball volumes are constant. This gives
us that $\zeta_1 = C_1/ (a\tau^d)$ and $\zeta_2 = C_2/ (a\tau^d)$.

Now, the NSW lemma can be restated as if $n = C_1/\tau^d(\log(C_2/\tau^d) + \log(1/\delta) )$
we recover the homology with probability at least $1 - \delta$. Notice that this means 
that the minimax risk $\leq \delta$.

A straightforward rearrangement of this gives us $$R_n \leq C_2/(a\tau^d) \exp(- na \tau^d/C_1)$$
for appropriate $C_1, C_2$. To bound the resolution we rewrite this as
$$R_n \leq \exp \left( - \frac{na\tau^d}{C_1} + \log \left( \frac{C_2}{a\tau^d} \right) \right)$$
One can verify that if $$\tau^d \leq C \frac{\log n \log(1/\epsilon)}{n}$$
for an appropriately large C, we have $R_n \leq \epsilon$ as desired.

\subsubsection{Proofs for the clutter noise case}
{\bf Lower bound}
\label{app:clutterlow}
This is a straightforward extension of the noiseless case. The densities are constructed in an
identical manner. The contribution to the densities from the clutter noise is identical in each case.
As in the analysis for the noiseless case we bound the total variation distance between the two densities. We have an additional factor of $\pi$ which is the mixture weight of the component corresponding to the 
density on the manifold.
$$TV(q_1,q_2) = \pi a \max(\vol(W_1),\vol(W_2)) \leq C_d \pi a \tau^d$$
Given this bound the calculations are identical to those in the noiseless case.

{\bf Upper bound}
\label{app:clutterup}
As a preliminary step we will need to clean the data to eliminate points that are far
away from the manifold. Our analysis will show that Algorithm~\ref{clean} will
 achieve this, with high probability. We will then show that 
taking a union of balls of the appropriate radius
around the remaining points will give us the correct homology,  with high probability.

Let $a = \inf_{x \in M} p(x)$, which is strictly positive by assumption.
Define,
$A = {\sf tube}_r({M})$ and
$B = \mathbb{R}^D - {\sf tube}_{2r}({M})$
where $r < \frac{(\sqrt{9} - \sqrt{8}) \tau}{2}$.
Following \cite{niyogi2010}, we define
$\alpha_s = \inf_{t \in A} Q(B_s(t))$ and
$\beta_s = \sup_{t \in B} Q(B_s(t))$
where $s = 2r$. Then
$\alpha_s   \geq  \frac{v_D s^D (1 - \pi)}{\mathrm{vol(Box)}} + \pi av_d r^d \cos^d \theta = \alpha$
and
$\beta_s  \leq  \frac{v_D s^D (1 - \pi)}{\mathrm{vol(Box)}} = \beta$
where $\theta = \sin^{-1}(\frac{r}{2\tau})$. The second term of the bound
on $\alpha_s$ follows in two steps: 
first observe that for any point $x$ in $A$, 
$B_s(x) \supseteq B_r(t)$ where
$t$ is the closest point on $M$ to $x$. Now, we use Lemma~\ref{lem:nsw_vol} to
bound $Q(B_r(t))$.

We will now invoke Algorithm {\sf CLEAN} on the data with
threshold 
$t = \left( \frac{v_D s^D (1 - \pi)}{\mathrm{vol(Box)}} +
\frac{\pi av_d r^d \cos^d \theta}{2} \right)$ 
and radius $2r$. Let $I$ be the set of vertices returned.

Define the events
${\cal E}_1 = \Biggl\{
\{X_i:\ i\in I\} \supseteq \{X_i\in A\}\ \ \ {\rm and}\ \ \ 
\{X_i:\ i\in I^c\} \supseteq \{X_i\in B\}\Biggr\}$
and
${\cal E}_2 = \Bigl\{ {M} \subset \bigcup_{i\in I} B_{2r}(X_i)\Bigr\}.$
We will show that ${\cal E}_1$ and ${\cal E}_2$ both
hold with high probability.

For ${\cal E}_1$ to hold, we need  $\beta$ to be not too
close to $\alpha$, in particular $\beta < \alpha/2$ will suffice. This happens
with probability 1, for $\tau$ small if $d < D$. By Lemma~\ref{bern:clean} in the Appendix, ${\cal E}_1$ happens with probability at least $1 - \delta/2$, provided that $n > 4 \kappa \log \kappa$, 
where 
$$
\kappa = \max \left(1 + \frac{200}{3\pi av_d r^d \cos^d (\theta)} 
\log \left( \frac{2}{\delta} \right) ,4 \right).
$$

Now we turn to ${\cal E}_2$.
Let $p_1,\ldots, p_N\in M$ be such that
$B_r(p_1),\ldots, B_r(p_N)$ forms a minimal covering of $M$.
From Lemma~\ref{sample}, we have that $N \leq \frac{\mathrm{vol}({M})}{\cos^d(\theta)v_d r^d}$.
Let $A_j = B_{r}(p_j)$.
Then
\begin{eqnarray*}
Q(A_j) & \geq & \frac{v_D s^D (1 - \pi)}{\mathrm{vol(Box)}} + 
\pi av_d r^d \cos^d (\theta) \\
& \geq & \pi av_d r^d \cos^d (\theta) \equiv \gamma.
\end{eqnarray*}
Using again Lemma~\ref{sample}, if
$n > \frac{1}{\gamma} 
\left(\log N + \log \left(\frac{2}{\delta}\right) \right)$,
then
with probability at least $1 - \delta/2$, 
each $A_i$ contains at least one sample point, 
and hence ${M} \subset \bigcup_{i\in I} B_{2r}(X_i)$, which implies that
${\cal E}_2$ holds.

Combining these we are now ready to again apply the main result from NSW. 
We restate this lemma in a slightly different form here.
\begin{lemma} 
\label{nsw:noisymain}
[NSW]
Let $S$ be a set of points in the tubular neighborhood of radius $R$
around $M$. Let $U = \bigcup_{x \in S} B_\epsilon(x)$. 
If $S$ is $R$-dense in $M$ then $\hat{\cal H}(U) = {\cal H}(M)$
for all $R < (\sqrt{9} - \sqrt{8}) \tau$, if $\epsilon = \frac{R + \tau}{2}$.
\end{lemma}

Combining the previously established facts with the lemma above
we obtain Lemma \ref{clutter:main} from the main paper.
Taking $r = (\sqrt{9} - \sqrt{8})\tau/4$ in that lemma, we can see that if $n \geq \frac{C_1}{\pi \tau^d}( \log \frac{C_2}{\tau^d} + \log(C_3/\epsilon))$ then we recover the correct homology with probability at least $1 - \epsilon$.

This is a sample complexity upper bound. Corresponding upper bounds on the minimax risk and resolution follow the arguments of the noiseless case.
\subsubsection{Proofs for the tubular noise case}
\label{app:tub}
{\bf Lower bound}
In this setting we get samples uniformly in a full dimensional tube around the manifold. We are interested in the case when $\sigma \leq C_0 \tau$ for a small constant $C_0$.

Let us denote the density $q_1$ at a point in the tube around $M_1$ by $\theta_1$ and the density $q_2$ around $M_2$ by $\theta_2$. Since, it is not straightforward to decide whether $\theta_1 \leq \theta_2$ or not we will need to consider both possibilities. We will show the calculations assuming $\theta_1 \leq \theta_2$ (the other calculation follows similarly).

Now, remember from the definition of total variation $TV = q_1(G) - q_2(G)$ where $G$ is the set where $q_1 > q_2$. We need an upper bound on total variation and so it suffices to use $TV \leq q_1(G^+) - q_2(G^-)$ where $G^+$ and $G^-$ are sets containing and contained in $G$ respectively.

Since, $\theta_1 < \theta_2$ we have $G$ is contained in the holes (of radius $4\tau$) of the two annuli, and $G$ contains a strip of width at least $2\tau - 2\sigma$ in these holes. These are $G^+$ and $G^-$.

We need to upper bound the mass under $q_1$ in $G^+$ and lower bound the mass under $q_2$ in $G^-$. We can now follow the a similar argument to the one made below (in the tubular noise upper bound) to obtain bounds on the various volumes.
 In each case, the volume of the tubular region is $\Omega(\vol(M) \sigma^{D-d})$, and both $M_1$ and $M_2$ have constant volume, in particular $c_1 \leq \vol(M) \leq C_1$. Giving us that the tubular region has volume $\Omega(\sigma^{D-d})$.
 
 It is also clear that both $G^+$ and $G^-$ have volumes that are $\Omega(\sigma^{D-d} \tau^d)$ (these can be calculated \emph{exactly} since they are cylindrical with no additional curvature but we will not need this here). Here we use that $\sigma$ is not too close to $\tau$ (and in particular is at most a constant fraction of $\tau$).
 
 Since $q_1$ and $q_2$ are both uniform in their respective tubes, it follows that
 $$TV(q_1, q_2) \leq \Omega \left( \frac{\sigma^{D-d} \tau^d}{\sigma^{D-d}} \right) = \Omega (\tau^d)$$
Notice, that we assumed $\theta_1 \leq \theta_2$ above. The other calculation is nearly identical and we will not reproduce it here.

{\bf Upper bound}
Denote by $M_\sigma$ the tube of radius $\sigma$ around $M$. Recall that we are interested in the case when $\sigma \ll \tau$, and $\epsilon = \tau/2$.
\begin{lemma}
\label{tubular_cover_lemma}
If $\epsilon \gg \sigma$ (in particular $\epsilon \geq 2\sigma$ will suffice)
$$k_\epsilon = \Omega(\epsilon^d).$$
\end{lemma}
\begin{proof}
For any $p \in M$, $$Q(B_\epsilon(p)) = \frac{vol(B_\epsilon(p) \cap M_\sigma)}{vol(M_\sigma)}.$$
We will prove the claim by deriving derive an upper bound on the denominator and a lower bound on the numerator using packing/covering arguments, both bounds holding uniformly in $p$.

{\bf Upper bound on $vol(M_\sigma)$}\\
We consider a covering of $M$ by
$\gamma$-balls of $d$ dimensions, and denote the number of balls required $N_\gamma$,
and the centers $\mathcal{C}_\gamma$. It is clear  $N_\gamma$ is bounded by the number of balls of radius $\gamma/2$ one can 
pack in $M$. A simple volume argument then gives $$N_\gamma \leq C \frac{vol(M)}{(\gamma/2)^d},$$
for some constant $C$.
Given this covering of $M$, it is easy to see that $\gamma + \sigma$ $D$-dimensional balls around 
each of the centers in $\mathcal{C}_\gamma$ covers the tubular region. Thus, we have $$vol(M_\sigma) \leq v_D N_\gamma (\gamma + \sigma)^D \leq v_D C \frac{vol(M)}{(\gamma/2)^d} (\gamma + \sigma)^D,$$ for any $\gamma$.
Selecting $\gamma = \sigma$, we have
$$vol(M_\sigma) \leq C_{D,d} vol(M) \sigma^{D-d}$$
for some constant $C_{D,d}$ depending on the manifold and ambient dimensions, independent
of $\sigma$.

{\bf Lower bound on $vol(B_\epsilon(p) \cap M_\sigma)$}\\
Define
\begin{eqnarray*}
A(p) = M \cap B_{\epsilon-\sigma}(p), \\
B(p) = M \cap B_\epsilon(p), \\
B_\sigma(p) = M_\sigma \cap B_\epsilon(p).
\end{eqnarray*}
Denote with $N_\sigma$  the number of points we can ``pack'' in $A(p)$ such that the distance between
any two points is at least $2\sigma$. Denote the points themselves by the
set $\mathcal{C}$.
Then, 
\begin{eqnarray*}
vol(B_\sigma) \geq N_\sigma v_D \sigma^D
\end{eqnarray*}
where $v_D$ is the volume of the unit ball in $D$-dimensions.
To see this just note that every point that is at most $\sigma$ away from any point in $\mathcal{C}$  is contained
in $B_\sigma$, and these sets are disjoint so the union of $\sigma$ balls around $\mathcal{C}$ is
contained in $B_\sigma$.

Now, to prove a lower bound on $N_\sigma$ we invoke some ideas from \cite{niyogi2008}. 
 Consider, the map $f$ described in Lemma 5.3 in \cite{niyogi2008}, which projects the manifold onto its tangent space, and observe its action on $A(p)$. It is clear by their discussion that
this map projects the manifold onto a superset of a ball of radius $(\epsilon - \sigma)\cos \theta$,
for $\theta = \sin^{-1} (\frac{\epsilon-\sigma}{2\tau})$.
In addition to being invertible, this map  is a projection, and only shrinks distances between points.
So if we can derive a lower bound on the number of points we can ``pack'' in this projection then it
is also a lower bound on $N_\sigma$.
Now, the set is just a ball in $d$-dimensions of radius $(\epsilon - \sigma)\cos \theta$. Using, the fact that
$2 \sigma$ balls around each of the points in $\mathcal{C}$ must cover this set a simple volume
argument shows
$$N_\sigma (2\sigma)^d \geq v_d ((\epsilon - \sigma)\cos \theta)^d,$$
i.e. $$N_\sigma \geq C_{D,d} \left( \frac{(\epsilon - \sigma) \cos \theta}{\sigma} \right)^d,$$
which gives a lower bound.

Putting the upper and lower bound together, we get
\begin{eqnarray*}
k_\epsilon & = & \inf_{p \in M} Q(B_\epsilon(p))  \\
& \geq & C_{D,d}^\prime \frac{1}{vol(M) \sigma^{D-d}} \left( \frac{(\epsilon - \sigma) \cos \theta}{\sigma} \right)^d \sigma^D \\
& = &  C_{D,d}^\prime \frac{ \left[ (\epsilon - \sigma) \cos \theta\right]^d}{vol(M)},
\end{eqnarray*}
for some quantity $C_{D,d}^\prime $, independent of $\sigma$.
\end{proof}

We will prove the following main lemma.
\begin{lemma}
Let $N_{\epsilon}$ be the $\epsilon$-covering number of the submanifold ${M}$. 
Let $U = \bigcup_{i=1}^n B_{\epsilon + \tau/2}(X_i)$.
Let $\hat{\cal H} = {\cal H}(U)$.
Then if $n > \frac{1}{k_{\epsilon}} \left( \log(N_{\epsilon}) + \log(1/\delta) \right)$,
$\mathbb{P}(\hat{\cal H} \neq {\cal H}(M)) < \delta$ as long as
$\sigma \leq \epsilon/2$ and $\epsilon < \frac{(\sqrt{9} - \sqrt{8})\tau}{2}$.
\end{lemma}
\begin{proof}
This is a straightforward consequence of Lemma \ref{nsw:noisymain} and Lemma \ref{sample}.
\end{proof}

\subsubsection{Proof of Theorem \ref{thm:additive} (additive case)}
\subsubsection*{Lower Bound}
From Lemma~\ref{conv} we see that convolution only decreases the total variation distance,
and so the lower bound for the noiseless case is still valid here.
\subsubsection*{Upper Bound}
We will again proceed by a similar argument to the clutter noise case.
Let $\sqrt{D} \sigma < r$, $R = 8r$ and $s = 4r$ and set 
$\alpha_s = \inf_{p \in A} Q(B_s(p))$ and $\beta_s = \sup_{p \in B} Q(B_s(p))$, where 
$A = \mathrm{tube}_r(M)$, $B = \mathbb{R}^D - \mathrm{tube}_R(M)$.

As in the clutter noise case, we will need the two events ${\cal E}_1$ and ${\cal E}_2$ to
hold with high probability.

We will use the following version of a common $\chi^2$ inequality, 
established by \cite{niyogi2010}.
\begin{lemma}
For a $D$-dimensional Gaussian random vector
$$\mathbb{P}(||\epsilon|| > \sqrt{T}) \leq (ze^{1-z})^{D/2}$$
where $z = \frac{T}{D\sigma^2}$
\end{lemma}

Using this inequality, 
$$\mathbb{P}(||\epsilon|| \geq 4r) \leq \left(16 \exp\{-15\}\right)^{D/2} \equiv t$$
and
$$\mathbb{P}(||\epsilon|| \geq 2r) \leq \left(4\exp\{-3\}\right)^{D/2} \equiv \gamma.$$
Observe that these are both constants. Next, it is easy to see that
$$\alpha_s \geq Q(B_{s - r}(p)) \geq a v_d r^d (\cos \theta)^d (1 - \gamma) \equiv \alpha,$$
where $\theta = \sin^{-1}(r/(2\tau))$, and
$$\beta_s \leq  v_D (8r)^D t \equiv \beta.$$
As in the clutter noise, we  need $\beta$ to be sufficiently smaller than $\alpha$ if we are to successfully clean the data. As we are interested in the case when $r$ is small, if $D > d$ then
we can take $\beta \leq \alpha/2$, while, if $D = d$ then we will need that the dimension is quite large (observe that both $\gamma$ and $t$ tend to zero rapidly rapidly as D grows).

We are now in a position to invoke the Lemma~\ref{bern:clean} to ensure ${\cal E}_1$ holds
with high probability for $n$ large enough.
Further, one can see that the mass of an $r/2$-ball close to manifold is at least $$ Q(A_i) \geq av_d(1-\gamma) (\cos \theta)^d (r/2)^d$$ for $\theta = \sin^{-1}(r/(4\tau))$. 
This quantity is also $O(r^d)$ as desired, and for $n$ large enough we can ensure ${\cal E}_2$ holds with high probability. Under the condition on $\sigma$, and $r$ we have
$r \leq \frac{(\sqrt{9} - \sqrt{8})\tau}{8}$. At this point we can invoke Theorem 5.1 from \cite{niyogi2010} to see that for $n \asymp^\ast \frac{1}{\tau^d}$ we recover the correct homology with high probability.

\subsubsection{Deconvolution}
{\bf Upper bound}
Recall, that the kernel $\Psi$ satisfies
\begin{equation}
\label{decon_tail}
\Psi \{ x: |x| \geq \epsilon\} \leq \gamma
\end{equation}
with $\epsilon$ and $\gamma$ being small constants that we will specify in our proof.

The starting point of our proof will be a uniform concentration result from Koltchinskii \cite{koltchinskii}.
\begin{lemma}
\label{uniformconc}
Consider the event $$ A = \{ \max_{x} |\hat{P}_{n}(B_{2\epsilon}(x)) - \hat{P}_\Psi(B_{2\epsilon}(x))| < \gamma \} $$
For any small constants $\epsilon$ and $\gamma$, there exists $q \in (0,1)$ such that $$P(A^c) \leq 4 q ^n $$ 
\end{lemma}
This lemma tells us that the deconvolved measure is uniformly close to a smoothed (by the kernel $\Psi$) version of the true density.

Our first step will be to draw 
$$m > \frac{1}{\omega} \left(\log l + \log \left(\frac{2}{\delta}\right) \right)$$
samples from $\hat{P}_n$, where $\omega = \inf_{x \in M} \hat{P}_n (B_{2\epsilon}(x))$, and $l$ is the $2\epsilon$ covering number of the manifold, and $\delta = 8q^n$.
Denote, this sample $Z$. We know that $l \leq \frac{\mathrm{vol}(M)}{\cos^d(\theta)v_d (2\epsilon)^d}$.

Let us first show that we can choose $\epsilon$ and $\gamma$ so that 
$\omega$ is at least a small positive constant.
\begin{eqnarray*}
\omega & = & \inf_{x \in M} \hat{P}_n (B_{2\epsilon}(x)) \\
& \geq &  \inf_{x \in M} P_\Psi (B_{2\epsilon}(x)) - \gamma 
\end{eqnarray*}
Notice that,
$$ P_\Psi(B_{2\epsilon}) \geq P(B_\epsilon) \Psi(x: |x| \leq \epsilon) $$
So, we have,
\begin{eqnarray*}
\omega & \geq & \inf_{x \in M} P(B_{\epsilon}(x)) (1 - \gamma) - \gamma
\end{eqnarray*}
Using the ball volume lemma we have,
\begin{eqnarray*}
\omega & \geq & a v_d \epsilon^d \cos^d \theta (1 - \gamma) - \gamma
\end{eqnarray*}
where $\theta = \sin^{-1}(\epsilon/2\tau)$.
Notice, that $\tau$ is a fixed constant, and 
$\epsilon$ and $\gamma$ are constants to be chosen appropriately.
 It is clear that 
for $\gamma \leq C_{d,\tau} \epsilon$, with $C_{d,\tau}$ small we have
$$\omega \geq c$$
for a small constant $c$ which depends on $\tau$,$d$ and our choices of $\epsilon$ and $\gamma$.

We now use the sampling lemma \ref{sample} to conclude that 
w.p. at least $1 - 4q^n$,
\begin{packed_enum}
\item The $m$ samples are $4 \epsilon$ dense around $M$.
\item $M \subset \cup_{i=1}^m B_{4\epsilon}(x_i)$
\end{packed_enum}

Our next step will be a cleaning step. This cleaning procedure differs from the Algorithm CLEAN
in that we use the deconvolved measure to clean the data.
In particular, we will remove all points from $Z$
for which $\hat{P}_n(B_{4\epsilon}(Z_i)) \leq 2 \gamma$. Denote the remaining points
by $W$.
Our estimator will then 
be constructed from $$H = \bigcup B_{\frac{5\epsilon + \tau}{2}} (W_i)$$

To analyze this cleaning procedure, we use the uniform concentration
lemma \ref{uniformconc} above, and consider the case when event $A$ happens.
\begin{enumerate}
\item \textbf{All points far away from $M$ are eliminated}: In particular, for any point $x$ if we have
$$ \mathrm{dist}(B_{4\epsilon}(x), M) \geq \epsilon $$ then the corresponding point is eliminated.

To see this is simple. We eliminated all points with deconvolved empirical mass $\hat{P}_{n}(B_{4\epsilon}) < 2\gamma$. Since, we are assuming event $A$ happened, we have for any remaining point $P_\Psi(B_{4\epsilon}) > \gamma$. Now, we have that 
\begin{equation*}
\Psi \{ x: |x| \geq \epsilon\} \leq \gamma
\end{equation*}
From this we see that some part of $B_{4\epsilon}$ must be within $\epsilon$ of $M$, and we have arrived at a contradiction.

\item \textbf{All points close to $M$ are kept}: In particular, for any point $x$ if 
$$\mathrm{dist}(x, M) \leq 2\epsilon $$then the corresponding point is kept.

We need to show $\hat{P}_n(B_{4\epsilon}(x)) \geq 2\gamma$. 
Notice, that $\hat{P}_n(B_{4\epsilon}(x)) \geq \hat{P}_n(B_{2\epsilon}(\pi(x)))$ where $\pi(x)$ is the projection of $x$ onto $M$. This quantity is just $\omega$. 

To finish, we need to show that we can choose $\epsilon$ and $\gamma$ such that $\omega \geq	2\gamma$. Since, $\omega \geq av_d \epsilon^d \cos^d \theta (1 - \gamma) - \gamma$ which as a function of $\gamma$ is continuous, bounded from below by a constant depending on $\tau$, $d$ and $\epsilon$ and monotonically increasing as $\gamma$ decreases we have for $\gamma$ small enough
$$\omega \geq 2 \gamma$$

 
%

\item \textbf{The set $H$ has the right homology}: 
We have shown that the cleaning eliminates all points outside a tube of radius $5\epsilon$, and further keeps all points in a tube of radius $2\epsilon$. From the sampling result we know the points that we keep are $4\epsilon$ dense and that $M \subset \cup_{i=1}^m B_{4\epsilon}(x_i)$. We can now apply lemma \ref{nsw:noisymain} to conclude that $H$ has the right homology provided $$\epsilon < \frac{(\sqrt{9} - \sqrt{8}) \tau}{5}$$Since $\tau$ is a fixed constant we can always choose $\epsilon$ small enough to satisfy this condition. To review, we need to select $\gamma$ and $\epsilon$ to satisfy three conditions
\begin{enumerate}
\item $\omega \geq av_d \epsilon^d \cos^d \theta (1 - \gamma) - \gamma$ has to be atleast a small positive constant.
\item $\omega \geq 2\gamma$
\item $\epsilon < \frac{(\sqrt{9} - \sqrt{8}) \tau}{5}$
\end{enumerate}
Each of these can be satisfied by choosing $\gamma$ and $\epsilon$ small enough.

Now, returning to $m$. We have $$m > \frac{1}{\omega} \left(\log l + \log \left(\frac{2}{\delta}\right) \right)$$
where $\omega = \inf_{x \in M} \hat{P}_n (B_{2\epsilon}(x))$, and $l$ is the $2\epsilon$ covering number of the manifold $l \leq \frac{\mathrm{vol}(M)}{\cos^d(\theta)v_d (2\epsilon)^d}$, and $\delta = 8q^n$. It is clear that all terms except those in $n$ are constant. In particular it is easy to see that $$m \geq Cn$$ for $C$ large enough is sufficient.

\end{enumerate}
From this we can conclude with probability at least $1 - 8q^n$ our procedure will construct an estimator with the correct homology. Since, $q \in (0,1)$ the success probability can be re-written as at least $1 - e^{-cn}$ for $c$ small enough. Together this gives us the deconvolution lemma from the main paper. 

\subsection{Additional technical lemmas}
\subsubsection{The cleaning lemma}
In this section we sharpen  Lemma 4.1 of \cite{niyogi2010}, also known as the A-B lemma, by using Bernstein's inequality instead of Hoeffding's inequality. This modification is crucial to obtain minimax rates.


\begin{lemma}
\label{bern:clean}
Let $\beta_s \leq \beta < \alpha/2 \leq \alpha_s/2$. If $n > 4\beta \log \beta$, where
\begin{eqnarray*}
\beta = \max \left(1 + \frac{200}{3\alpha} \log \left( \frac{1}{\delta} \right) ,4 \right),
\end{eqnarray*}
then procedure {\sf CLEAN}$(\frac{\alpha + \beta}{2})$  will remove all points in region $B$ and keep all points in 
region $A$ with
probability at least $1 - \delta$.
\end{lemma}
\begin{proof}
We use the notation established in section \ref{sec:clutter.upper}.  We first analyze the set $A$. 

For a point $X_i$ in $A$, let $q = q(i) = Q(B_s(X_i))$, and define, 
$$Z_j = \mathbb{I}(X_j \in B_s(X_i)), \quad j \neq i,$$ where $ \mathbb{I}$ denotes the indicator function. Notice that the random variables $\{ Z_j, j \neq i\}$ are independent Bernoulli with common mean $q$.

We will consider two cases.

Case 1:  $\alpha \leq q \leq 2\alpha$.\\
 Notice that if 
$$q - \frac{1}{n-1} \sum_{j \neq i} Z_j \leq \frac{\alpha}{4}$$
the point $X_i$ will not be removed. By Bernstein's inequality, the probability that $X_i$ will instead be removed is 
\begin{eqnarray*}
\mathbb{P}\left(q - \frac{1}{n-1} \sum_{j \neq i} Z_j \geq \frac{\alpha}{4}\right) & \leq & \exp\left\{ -\frac{1}{2} \frac{(n-1)(\alpha/4)^2}{2\alpha + \alpha/12} \right\} \\
& \leq & \exp \left\{ -\frac{3}{200} (n-1) \alpha \right\}.
\end{eqnarray*}

Case 2: $q > 2\alpha$.\\
In this case if 
$$q - \frac{1}{n-1} \sum_{j \neq i} Z_j \leq q - \frac{3\alpha}{4}$$
the point $X_i$ \emph{will} be removed. Another application of Bernstein's inequality yields

\begin{eqnarray*}
\mathbb{P} & & \left(q - \frac{1}{n-1} \sum_{j \neq i} Z_j \geq q - \frac{3\alpha}{4}\right)  \\
 \leq &  & \exp\left\{ -\frac{1}{2} \frac{(n-1)(q - 3\alpha/4)^2}{q + (q - 3\alpha/4)/3} \right\} \\
 \leq  &  &\exp \left\{ -\frac{1}{2} (n-1) \left[ \frac{q}{2} + \frac{9\alpha^2}{32p} - \frac{3\alpha}{4} \right]\right\} \\
 \leq & & \exp \left\{ - \frac{(n-1)\alpha}{8} \right\}.
\end{eqnarray*}

Now, consider a point $X_i$ in the region B, and define $q$ and the $Z_j$s in an identical way.
This time if $$\frac{1}{n-1} \sum_{j \neq i} Z_j - q \leq \frac{\alpha}{4},$$ the point $X_i$ will not be removed.
By Bernstein's inequality,
\begin{eqnarray*}
\mathbb{P}\left(\frac{1}{n-1} \sum_{j \neq i} Z_j - q \geq \frac{\alpha}{4}\right) & \leq & \exp\left\{ -\frac{1}{2} \frac{(n-1)(\alpha/4)^2}{\alpha/2 + \alpha/12} \right\} \\
& \leq  &\exp \left\{ -\frac{3}{56} (n-1) \alpha \right\}
\end{eqnarray*}

Putting all the pieces together, we obtain that the cleaning procedure succeeds on all points  with probability at least $n \exp \left\{ -\frac{3}{200} (n-1) \alpha \right\}$. This requires,
\begin{eqnarray*}
n - 1 & > & \frac{200}{3\alpha} \left( \log n + \log \left( \frac{1}{\delta} \right) \right)~~\mathrm{i.e.} \\
n & > & 1 + \frac{200}{3\alpha} \log \left( \frac{1}{\delta} \right) +  \frac{200}{3\alpha} \log n
\end{eqnarray*}
If $\delta < 1/2$, then $1 + \frac{200}{3\alpha} \log \left( \frac{1}{\delta} \right) > \frac{200}{3\alpha}$, so it
is enough to solve $$n > x + x \log n$$with $x = 1 + \frac{200}{3\alpha} \log \left( \frac{1}{\delta} \right)$. The result of the lemma follows.
\end{proof}

\subsubsection{Convolution only decreases total variation}
\begin{lemma}
\label{conv}
Let $P$ and $Q$ two probability measures in $\mathbb{R}^D$ with common dominating measure $\mu$. Then,
\[
{\sf TV}(P \star \Phi,Q \star \Phi) \leq C_\phi {\sf TV}(P,Q).
\]
where $\star$ denotes deconvolution and $\Phi$ is a probability measure on $\mathbb{R}^D$.

\end{lemma}
\begin{proof}
This is a standard result, but we provide a proof for completeness.
Let $p \star \phi$ denote the  Lebesgue density of the probability distribution $P \star \Phi$, i.e.
\[
p \star \phi(z) = \int \phi(z-x) p(x) d \mu(x), \quad z \in \mathbb{R}^{D}.
\] 
 Similarly, $q \star \phi$  denotes the analogous quantity for $Q \star \Phi$. Then,
\begin{eqnarray*}
2 {\sf TV}(P \star \Phi,Q \star \Phi) & = &  \int_{\mathbb{R}^D} \left| p \star \phi(z)  - q \star \phi(z) \right| dz \\
& = &  \int_{\mathbb{R}^D} \left| \int \phi(z-x) p(x) d \mu(x) \right. \\
& &  \left. - \int \phi(z-x) p(x) d \mu(x) \right| dz  \\
& = & \int_{\mathbb{R}^D} \left| \int \phi(z-x) ( p(x) \right. \\
& & \left. - q(x)) d \mu(x)   \right| dz   \\
& \leq & \int_{\mathbb{R}^D}  \int \left| \phi(z-x) ( p(x) \right. \\
& & \left. - q(x)) \right| d \mu(x)    dz   \\ 
& \leq &  \int \int_{\mathbb{R}^D} \phi(z-x)  dz \left| p(x) - q(x) \right| d \mu(x)     \\ 
 & = & \int  \left| ( p(x) - q(x) \right| d \mu(x)  \\
 & = & 2  {\sf TV}(P,Q)
\end{eqnarray*}
\end{proof}

\end{document}